%% file: iclr2027_conference.tex
\documentclass{article} 
\usepackage{iclr2027_conference,times}

\input{math_commands.tex}

\usepackage{amsthm}
\newtheorem{proposition}{Proposition}
\usepackage[dvipsnames,svgnames]{xcolor}
\usepackage{hyperref}
\hypersetup{
  colorlinks=true,
  citecolor=SteelBlue,
}
\usepackage{url}
\usepackage{graphicx}
\usepackage{float}

\usepackage{graphicx,booktabs,array,multirow}
\usepackage{colortbl}
\usepackage{pifont}
\definecolor{datasetbg}{RGB}{243,240,233}
\definecolor{oursbg}{RGB}{237,237,237}

\title{Where and When to Force: Routed Forcing for Streaming Avatars}

\author{
\textbf{Zihan Su}$^{1,2}$\thanks{Equal contribution.} \quad
\textbf{Siwen Lu}$^{1}$\footnotemark[1] \quad
\textbf{Junhao Zhuang}$^{2}$\thanks{Corresponding authors.} \quad
\textbf{Zeyue Xue}$^{2}$ \quad
\textbf{Haoyang Huang}$^{2}$ \\
\textbf{Guanghao Li}$^{1}$ \quad
\textbf{Xiaofeng Tan}$^{3}$ \quad
\textbf{Chun Yuan}$^{1}$\footnotemark[2] \quad
\textbf{Nan Duan}$^{2}$ \\[6pt]
$^1$Tsinghua University \quad
$^2$Joy Future Academy, JD \quad
$^3$Southeast University
}

\iclrfinalcopy
\begin{document}

\maketitle

\begin{abstract}
Audio-driven streaming avatar generation requires real-time synthesis of speech-synchronized videos with dynamic and diverse motion. Self Forcing uses Distribution Matching Distillation (DMD) to distill bidirectional video diffusion models into causal, few-step generators for real-time streaming. However, DMD minimizes a reverse KL divergence, which is inherently mode-seeking: it causes the student to discard high-dynamic modes and collapse onto static outputs, compressing both dynamics and diversity of generated videos. We find that this collapse is \emph{region-heterogeneous}: person regions involving pose and gesture variations suffer the largest diversity loss, the audio-driven mouth region shows a small loss, and the background remains nearly stable. Based on this observation, we propose \textbf{Routed Forcing}, which routes the distillation objective by semantic region and noise stage to improve dynamics and diversity while preserving visual quality. Specifically, (1)~\emph{Where to Force: Semantic-Region Routing} applies Data-Forcing Distillation (DFD), which supervises the student with real videos, to the person region where diversity collapse is most severe, while retaining DMD for the mouth and background to preserve lip synchronization and scene stability. (2)~\emph{When to Force: Noise-Stage Routing} activates DFD at high noise stages, where real video serves as effective supervision to inject diverse and dynamic motion patterns. At low noise stages, DMD is used to refine details, avoiding blur and artifacts from spatial differences between real video and student-generated video. Experiments show that Routed Forcing improves dynamics by up to 45\% and diversity by 7--25\% over Self Forcing, while preserving video quality and lip synchronization.
\end{abstract}

\section{Introduction}

Audio-driven portrait video generation produces videos whose motion follows input speech~\citep{suwajanakorn2017synthesizing,zhou2020makeittalk,zhang2023sadtalker,tian2024emo}. Recent methods further extend this paradigm toward causal, real-time streaming generation for interactive applications~\citep{zhen2025teller,zhen2026soulxliveact,cho2026lipforcing,xiao2025knotforcing}.
This task requires the system to satisfy two simultaneous demands: real-time streaming output at interactive frame rates, and temporally rich dynamics including natural upper-body motion, responsive head turns, and coordinated gestures.
Large bidirectional video diffusion models~\citep{yang2024cogvideox,wan2025wan,gan2025omniavatar,wang2025fantasytalking} can already produce high-quality portrait videos, but their multi-step inference cost cannot meet the latency requirements of real-time interaction.
Self Forcing~\citep{huang2025selfforcing} leverages Distribution Matching Distillation (DMD)~\citep{yin2024dmd,yin2024dmd2} to distill these models into causal, few-step streaming generators, enabling real-time generation.

\begin{figure}[t]
  \centering
  \includegraphics[width=\linewidth]{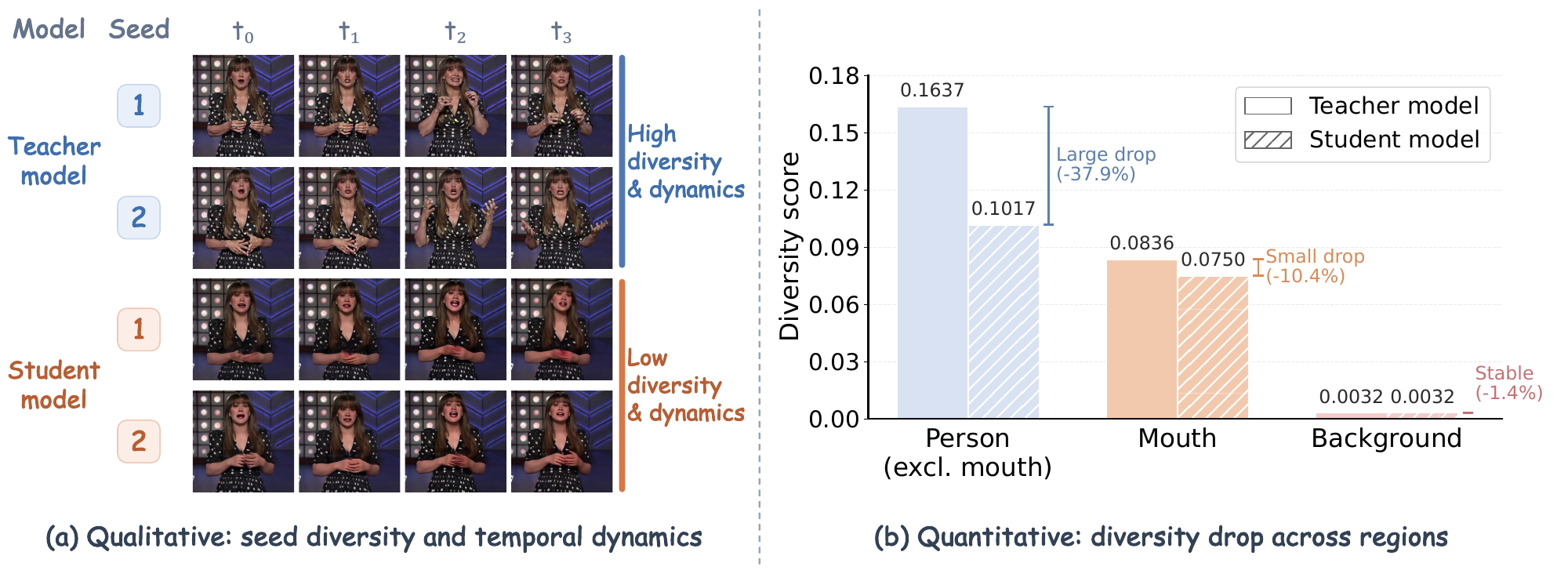}
  \vspace{-1.7em}
  \caption{\textbf{DMD compresses dynamics and diversity, and the collapse is region-heterogeneous.} We use OmniAvatar as teacher model and LiveTalk, distilled from OmniAvatar via DMD, as student model. (a)~Given the same first frame and audio, the teacher produces diverse expressions and dynamic gestures across two seeds, while the student generates similar static trajectories. (b)~Regional diversity (mean pairwise LPIPS across four seeds over 100 first-frame-and-audio pairs from SpeakerVid) reveals that the non-mouth person region suffers the largest drop, the mouth region shows a small drop, and the background remains nearly stable.}
\label{fig:teaser}
\end{figure}

However, DMD compresses the dynamics and diversity of generated videos while reducing sampling steps.
As shown in Figure~\ref{fig:teaser}a, given the same first frame and audio with different random seeds, the original teacher produces varied yet plausible expressions and dynamic gestures, while the DMD-trained student generates highly similar motion trajectories with little hand movement.
In other words, the teacher's diverse and dynamic patterns have been compressed into a single, static output.
The root cause is that DMD minimizes a reverse KL divergence, which is inherently mode-seeking~\citep{yin2024dmd}: it penalizes the student for placing mass outside the teacher's support but imposes no penalty for missing teacher modes.
\textit{A direct consequence of this diversity collapse is that the student discards high-dynamic modes and falls back to the simplest low-dynamic modes~\citep{liu2025dynaforcing}, resulting in generated outputs that consistently lack natural motion variation, especially harmful for avatars that demand lifelike motion.}

Furthermore, we find that this diversity collapse is not uniform across all regions.
We measure diversity across the mouth, non-mouth person region, and background over 100 initial conditions, as shown in Figure~\ref{fig:teaser}b.
The non-mouth person region has many plausible modes for the same speech input, including varied head motions, expressions, and gestures, and suffers the largest diversity loss.
The mouth region is strongly constrained by audio, with small diversity loss.
The background is nearly deterministic and remains stable.
We term this phenomenon \emph{region-heterogeneous collapse}: different semantic regions suffer different degrees of diversity loss under DMD distillation.
This regionally heterogeneous structure indicates that addressing DMD-induced diversity collapse should be differentiated according to regional characteristics.

Recently, Data-Forcing Distillation (DFD)~\citep{chen2026dfd} was proposed to address the diversity collapse caused by DMD.
It introduces real-data supervision by querying the real score on condition-matched real video instead of the student rollout video, injecting the high-dynamic, diverse motion patterns from real video into the student.
DFD can improve dynamics and diversity, but it applies uniformly across all regions and noise stages, ignoring the regional differences and noise-stage differences.
Injecting diversity into regions such as the mouth and background instead disrupts lip accuracy and scene stability. Moreover, at low noise stages, layout differences between real video and the student rollout video (hand position, head angle, body pose) are large, misaligning DFD's gradient with the student's generation and introducing blur and artifacts.

Motivated by the above findings, we propose \textbf{Routed Forcing}, which routes DFD along two dimensions to apply it precisely where and when it is needed.
\textbf{(1)~Where to Force: Semantic-Region Routing.}
Based on regional diversity differences, Routed Forcing confines DFD to the non-mouth person region where collapse is most severe, while mouth and background regions retain DMD to preserve lip accuracy and scene stability.
\textbf{(2)~When to Force: Noise-Stage Routing.}
Routed Forcing further refines DFD application in the non-mouth person region by using DFD and DMD at different noise stages.
At high noise stages, real video and the student rollout video are close in noisy space, so DFD is used to inject diverse and dynamic patterns from real video. At low noise stages, DMD is used instead to refine details on the student's own trajectory, avoiding blur and artifacts from cross-trajectory interference.

Experiments show that Routed Forcing improves dynamics by up to 45\% and diversity by 7--25\% over Self Forcing, while preserving video quality and lip synchronization.
Ablation studies confirm the contribution of each routing dimension: semantic-region routing avoids the damage that applying DFD to all regions inflicts on lip accuracy and scene stability, noise-stage routing further eliminates the blur introduced by layout differences at low noise stages, and their combination yields the best overall performance.
Our contributions are summarized as follows:

\begin{itemize}
\item We propose Routed Forcing for streaming avatar generation, a simple yet effective training strategy which routes the distillation objective by semantic region and noise stage to improve dynamics and diversity while preserving visual quality.
\item We discover \emph{region-heterogeneous collapse} in DMD: diversity loss concentrates in specific semantic regions rather than uniformly across the frame.
\item We propose semantic-region routing, which confines DFD to the non-mouth person region, where collapse is most severe, while retaining DMD for the mouth and background to preserve lip accuracy and scene stability.
\item We propose noise-stage routing, which activates DFD only at high noise stages to inject diverse and dynamic motion patterns from real video, and uses DMD at low noise stages to refine details and avoid blur and artifacts.
\end{itemize}

\section{Related Work}

\textbf{Audio-driven avatar generation.}
Early methods animate facial landmarks or 3D morphable models from audio~\citep{suwajanakorn2017synthesizing,chen2019hierarchical,zhou2020makeittalk,thies2020neural,zhang2023sadtalker}, while recent methods directly generate video frames with diffusion models.
AniPortrait~\citep{wei2024aniportrait} combines facial landmarks with a diffusion model to produce high-quality portrait animation, FantasyTalking~\citep{wang2025fantasytalking} further improves expression and motion quality through a diffusion transformer, and OmniAvatar~\citep{gan2025omniavatar} introduces adaptive body animation into a bidirectional diffusion transformer for high-fidelity portrait generation.
However, these methods require many denoising steps per frame and cannot meet the demands of real-time interaction.
LiveTalk~\citep{chen2025livetalk} achieves real-time streaming by distilling a bidirectional teacher into a causal, few-step student via DMD, but the distillation process introduces diversity collapse.
\textit{Routed Forcing maintains real-time inference while notably improving dynamics and diversity without sacrificing video quality or lip synchronization.}

\textbf{Diffusion model distillation.}
Distilling multi-step diffusion models~\citep{ho2020ddpm,tan2026consistentrft,tan2026pec,su2026safe,su2026generation} into few-step generators is key to enabling real-time inference.
Progressive distillation~\citep{salimans2022progressive} achieves acceleration by iteratively halving the number of sampling steps, while consistency distillation~\citep{song2023consistency,luo2023latent} learns a consistency mapping along the PF-ODE trajectory for single- or few-step generation.
DMD~\citep{yin2024dmd,yin2024dmd2} matches the score functions of the student and teacher distributions, evaluated entirely on the student's own samples, producing a mode-seeking (reverse-KL-like) gradient.
To mitigate the mode collapse of DMD, DFD~\citep{chen2026dfd} queries the frozen teacher score on noised real data, so that teacher modes unvisited by the student still provide supervision.
\textit{In streaming avatar generation, DMD causes diversity collapse, while global DFD introduces blur at low noise stages. Routed Forcing restricts DFD to the person region at high noise stages, improving diversity where collapse is most severe while avoiding lip-sync degradation.}

\textbf{Streaming video generation.}
Self Forcing~\citep{huang2025selfforcing} feeds the student's own generated output as the conditioning for subsequent frames and applies DMD on top to align the student and teacher distributions, achieving causal, few-step streaming video generation.
Building on this paradigm, several works improve the training strategy.
Reward Forcing~\citep{lu2025rewardforcing} applies exponential reward weighting to the DMD gradient, steering the student distribution toward high-reward samples.
DynaForcing~\citep{liu2025dynaforcing} replaces the rollout starting point with a noised real latent and combines dynamics-aware reward weighting to alleviate dynamic collapse.
Causal Forcing~\citep{zhu2026causalforcing} further improves distillation quality by bridging the architectural gap between bidirectional teacher and causal student.
\textit{These methods focus on improving the rollout, reward signal, or initialization, whereas Routed Forcing explores a different angle by differentiating the distillation objective at the level of semantic regions and noise stages.}

\section{Routed Forcing}
\label{sec:method}

We propose Routed Forcing, a simple and effective training strategy that routes the distillation objective by semantic region and noise stage to improve diversity and dynamics while preserving generation quality.
We first introduce preliminaries in \S\ref{sec:background}, then present the region-heterogeneous collapse phenomenon in \S\ref{sec:observation}, semantic-region routing in \S\ref{sec:region-routing}, and noise-stage routing in \S\ref{sec:noise-routing}.
The overall pipeline is illustrated in Figure~\ref{fig:rf-method}.

\begin{figure}[t]
\centering
\includegraphics[width=\linewidth]{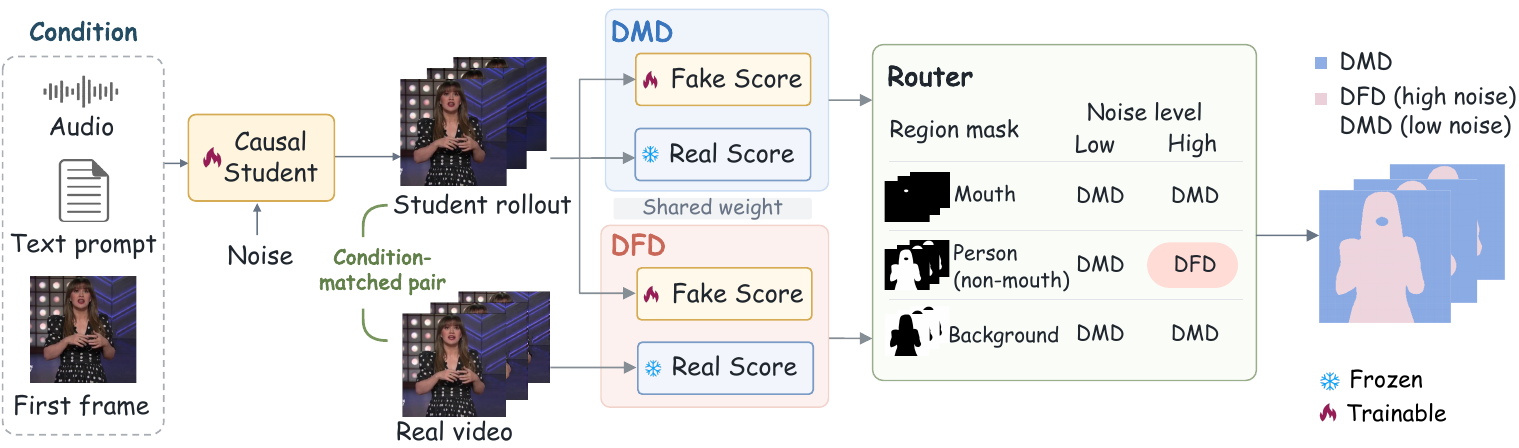}
\vspace{-1.7em}
\caption{\textbf{Routed Forcing routes distillation by semantic region and noise stage.} The causal student generates streaming video from a first frame, audio, and text prompt. DMD queries both the fake score and real score on the student rollout video, while DFD queries the real score on real video with the same condition to inject diverse and dynamic motion patterns from real data. The router selects DFD only for the non-mouth person region at high noise stages and DMD elsewhere.}
\label{fig:rf-method}\end{figure}

\subsection{Preliminaries}
\label{sec:background}

\textbf{Self Forcing.}
Given a reference image, audio, and text prompt as condition $c$, the teacher is a bidirectional diffusion transformer that generates high-quality video via multi-step denoising but cannot meet the demands of real-time inference.
Self Forcing~\citep{huang2025selfforcing} distills it into a causal student that generates frames or chunks autoregressively, enabling streaming generation.
During training, the student performs forward inference on its own rollout and aligns its output distribution with the teacher via DMD.

\textbf{Distribution Matching Distillation (DMD).}
DMD~\citep{yin2024dmd} trains the student by minimizing the reverse KL divergence between the student-induced distribution $p_\theta$ and the teacher data distribution $p_{\text{data}}$ across noise levels:
\begin{equation}
\mathcal{L}_\text{DMD} = \mathbb{E}_t \left[ D_\text{KL}( p_{\theta,t} \| p_{\text{data},t} ) \right].
\label{eq:dmd-loss}
\end{equation}
The gradient takes the form of the difference between a fake score and a teacher score, both evaluated on the student's own samples:
\begin{equation}
g_{\text{DMD}}(t) = w_t \left[ \hat{x}_{\text{fake}}(z_t^s, c, t) - \hat{x}_T(z_t^s, c, t) \right],
\label{eq:dmd}
\end{equation}
where $z_t^s = \alpha_t x^s + \sigma_t \epsilon$ is the noised student rollout, $x^s$ is the clean latent generated by the student, $\hat{x}_T$ is the frozen teacher denoiser, $\hat{x}_{\text{fake}}$ is a learned fake-score model tracking the evolving student distribution, and $w_t$ is a noise-level-dependent weighting coefficient.
This gradient is mode-seeking: it penalizes the student for generating outside the teacher's support but provides no signal for teacher modes the student has not visited.

\textbf{Data-Forcing Distillation (DFD).}
To mitigate the mode collapse of DMD, DFD~\citep{chen2026dfd} replaces the teacher's input with a noised real video $z_t^r = \alpha_t x^r + \sigma_t \epsilon$, where $x^r$ is condition-matched real video sharing the same noise $\epsilon$ as the student:
\begin{equation}
g_{\text{DFD}}(t) = w_t \left[ \hat{x}_{\text{fake}}(z_t^s, c, t) - \hat{x}_T(z_t^r, c, t) \right].
\label{eq:dfd}
\end{equation}
Because $x^r$ is drawn from the full data distribution, teacher modes that the student has never visited still contribute supervision.
In practice, DFD starts from a converged DMD checkpoint and switches between DMD and DFD gradients with 50\% probability at each training step.

\subsection{Observation: Region-Heterogeneous Collapse}
\label{sec:observation}

The mode-seeking nature of DMD leads to diversity loss after distillation, and we find that this loss is not uniform across all regions.
We define three semantic regions: mouth, non-mouth person, and background. We quantify diversity for each semantic region by measuring the mean pairwise LPIPS across different random seeds. Results measured with $L_1$ are provided in Figure~\ref{fig:l1_diversity} in the appendix.

As shown in Figure~\ref{fig:teaser}a, the teacher produces diverse gesture variations under the same condition with different seeds, such as raising and waving hands, while the student's two seeds yield nearly identical and static trajectories, with hands showing almost no change and remaining consistent with the first frame. Figure~\ref{fig:collapse_curve} in the appendix further confirms this trend: diversity decreases during DMD training.
The quantitative analysis in Figure~\ref{fig:teaser}b further reveals the region-heterogeneous nature of this collapse: the non-mouth person region suffers the largest diversity drop of 37.9\%, far exceeding the mouth (10.4\%), while the background remains nearly stable (1.4\%).

We term this phenomenon \emph{region-heterogeneous collapse}: different semantic regions suffer different degrees of diversity loss under DMD distillation.
The non-mouth person region has many plausible modes for the same speech input, including varied head motions, expressions, and gestures, and suffers the largest diversity loss.
The mouth region is strongly constrained by audio, with small diversity loss.
The background is nearly deterministic and remains stable.
This regionally heterogeneous structure motivates subsequent region-level differentiation in the distillation objective.

\subsection{Where to Force: Semantic-Region Routing}
\label{sec:region-routing}

Based on the observation in \S\ref{sec:observation}, Routed Forcing assigns different distillation objectives to different semantic regions according to their collapse severity.
Routed Forcing uses masks to assign each spatial location to three semantic regions: mouth~($M_m$), non-mouth person~($M_p$), and background~($M_b$), where $M_m + M_p + M_b = 1$.
The three regions have distinct distillation needs. The non-mouth person region (head motion, expression, gestures) suffers the most severe diversity collapse and is where the real-data supervision of DFD is most needed.
The mouth region is strongly constrained by audio, with limited diversity collapse, retaining DMD to preserve lip accuracy.
The background region is nearly deterministic, with minimal diversity collapse, and similarly uses DMD to maintain scene stability.

Routed Forcing therefore introduces DFD only in the non-mouth person region.
Following the original DFD training strategy, this region switches between $g_{\text{DMD}}(t)$ and $g_{\text{DFD}}(t)$ with 50\% probability at each step.
The expected gradient after region routing is:
\begin{equation}
g = M_p \cdot \tfrac{1}{2}\bigl[g_{\text{DMD}}(t) + g_{\text{DFD}}(t)\bigr] + (1 - M_p) \cdot g_{\text{DMD}}(t).
\label{eq:region-routing}
\end{equation}

However, region routing alone is not sufficient: even when restricted to the non-mouth person region, applying DFD across all noise stages still introduces artifacts at low noise stages.
This motivates further routing along the noise-stage dimension.

\subsection{When to Force: Noise-Stage Routing}
\label{sec:noise-routing}

DFD replaces the real score's input $z_t^s$ with the noised real video $z_t^r$, where $z_t^s = \alpha_t x^s + \sigma_t \epsilon$, $z_t^r = \alpha_t x^r + \sigma_t \epsilon$, $x^s$ is the clean latent generated by the student, $x^r$ is the condition-matched real video latent, $\sigma_t$ is the noise coefficient, and both share the same noise $\epsilon$.
Since the shared noise cancels, the noise-normalized distance between $z_t^s$ and $z_t^r$ is:
\begin{equation}
\frac{\|z_t^s - z_t^r\|}{\sigma_t} = \sqrt{\text{SNR}_t} \cdot \|x^s - x^r\|,
\label{eq:snr-distance}
\end{equation}
where $\text{SNR}_t = \alpha_t^2 / \sigma_t^2$ is the signal-to-noise ratio at timestep $t$.

At high noise stages ($\text{SNR}_t$ small), real video and student rollout video are close in noise space, with small cross-trajectory interference.
Denoising at this stage determines coarse-grained semantic trajectories, and introducing real video here injects the rich and dynamic motion patterns into the student, improving the dynamics and diversity.
At low noise stages ($\text{SNR}_t$ large), the DFD update approaches a direct fit to the paired real trajectory, and cross-trajectory differences (hand positions, head angles, body poses) can blur structural details.
Denoising at this stage controls fine details (texture, hand details, facial identity consistency) and should be guided by DMD on the student's own generation trajectory.
We provide a detailed mathematical analysis of why DFD should be restricted to high noise stages in Appendix~\ref{app:noise}.

Based on the above analysis, instead of the 50\% random mixing used in original DFD (Eq.~\ref{eq:region-routing}), we propose Noise-Routed DFD (NR-DFD), which deterministically restricts DFD to high noise stages based on a noise threshold $\tau$:
\begin{equation}
g_{\text{NR-DFD}}(t) = \begin{cases} g_{\text{DFD}}(t) & t \geq \tau \quad (\text{high noise}) \\ g_{\text{DMD}}(t) & t < \tau \quad (\text{low noise}). \end{cases}
\label{eq:nr-dfd}
\end{equation}

Combining region routing and noise routing, the full Routed Forcing gradient is:
\begin{equation}
g_{\text{RF}} = M_p \cdot g_{\text{NR-DFD}}(t) + (1 - M_p) \cdot g_{\text{DMD}}(t).
\label{eq:rf-total}
\end{equation}

By combining semantic-region routing and noise-stage routing, Routed Forcing precisely applies DFD's real-data supervision where and when it is needed, improving diversity and dynamics while avoiding the degradation in lip accuracy and visual quality caused by global DFD.

\section{Experiments}

\subsection{Experiment Setup}
\label{sec:experiment-setup}

\textbf{Implementation details.}
We implement Routed Forcing based on OmniAvatar~\citep{gan2025omniavatar}. The student is a causal OmniAvatar-1.3B initialized from LiveTalk~\citep{chen2025livetalk} weights. The frozen teacher (real score) is a bidirectional OmniAvatar-14B, and the fake-score model is a bidirectional OmniAvatar-1.3B. We sample 10,000 videos from SpeakerVid~\citep{zhang2025speakervid} as the training set. Following DFD~\citep{chen2026dfd}, we first train 1,000 steps of Self Forcing with DMD to provide a high-quality starting point. Routed Forcing training is then applied for 200 steps. At inference, we generate $512\times512$ videos of 81 frames at 16\,fps, corresponding to 21 latent frames with a causal block size of 3 and four sampling steps per block. Routed Forcing employs Top-50\% noise routing, activating DFD only at the top 50\% noise timesteps for the non-mouth person region. Person masks are extracted from real videos using DeepLabV3~\citep{chen2017deeplabv3}, and mouth masks are obtained from student rollouts using InsightFace~\citep{deng2020retinaface}. We discuss the mask design choice in Appendix~\ref{app:mask}. More implementation details are provided in Appendix~\ref{app:hyperparams}.

\textbf{Baselines.}
We set up two groups of comparisons. (1)~\textit{Controlled training strategy comparison}: Self Forcing~\citep{huang2025selfforcing} (DMD only), Data-Forcing Distillation (DFD)~\citep{chen2026dfd} (50\% probability DMD, 50\% probability DFD at each step), DP-DMD~\citep{lu2025dpdmd}, DynaForcing~\citep{liu2025dynaforcing}, and Reward Forcing~\citep{lu2025rewardforcing} (using the reward model from DynaForcing). All methods share the same initialization checkpoint, training data and random seeds, differing only in the algorithm design. (2)~\textit{Public model comparison}: We also compare with audio-driven avatar models of comparable scale: LiveTalk~\citep{chen2025livetalk}, FantasyTalking~\citep{wang2025fantasytalking}, AniPortrait~\citep{wei2024aniportrait}, and OmniAvatar~\citep{gan2025omniavatar}.

\textbf{Evaluation metrics.}
We randomly sample 100 videos from SpeakerVid~\citep{zhang2025speakervid}, with no overlapping identities with the training set, and 100 from AVSpeech~\citep{ephrat2018avspeech} as test sets. We report ten metrics in four groups. (1)~\textit{Video quality}: FID~\citep{heusel2017gans} and FVD~\citep{unterthiner2018fvd}. (2)~\textit{Lip sync}: Sync-C and Sync-D~\citep{chung2016syncnet}. (3)~\textit{Dynamics}: RAFT-Motion~\citep{teed2020raft} (pixel-level motion intensity) and DINO-Temporal~\citep{oquab2024dinov2} (semantic-level temporal variation). (4)~\textit{Diversity}: frame-wise mean pairwise distance across four seeds using CLIP~\citep{radford2021clip}, DINOv2, LPIPS~\citep{zhang2018lpips}, and L1. For comparison with audio-driven avatar models, we additionally report FPS and first-frame latency.

\subsection{Main Results}

\begin{table}[!t]
\centering
\caption{\textbf{Controlled comparison of training strategies.} All methods share the same initialization checkpoint, training data and random seeds, differing only in the algorithm design.}
\label{tab:forcing-comparison}
\vspace{6pt}
\setlength{\tabcolsep}{4pt}
\renewcommand{\arraystretch}{1.25}
\resizebox{\textwidth}{!}{%
\begin{tabular}{lcccccccccc}
\toprule[0.9pt]
\multirow{2}{*}{\raisebox{-0.5ex}{\textbf{Method}}} & \multicolumn{2}{c}{\textbf{Video Quality}} & \multicolumn{2}{c}{\textbf{Lip Sync}} & \multicolumn{2}{c}{\textbf{Dynamics} ($\times100$)} & \multicolumn{4}{c}{\textbf{Diversity} ($\times100$)} \\
\cmidrule(lr){2-3}\cmidrule(lr){4-5}\cmidrule(lr){6-7}\cmidrule(l){8-11}
& FID $\downarrow$ & FVD $\downarrow$ & Sync-C $\uparrow$ & Sync-D $\downarrow$ & RAFT $\uparrow$ & DINO-T $\uparrow$ & CLIP $\uparrow$ & DINO $\uparrow$ & LPIPS $\uparrow$ & L1 $\uparrow$ \\
\midrule
\multicolumn{11}{c}{\textbf{\textit{SpeakerVid}}} \\
\midrule
Self Forcing & 27.17 & 320.91 & 4.03 & 10.97 & 5.62 & 1.07 & 3.07 & 1.37 & 4.96 & 2.14 \\
DFD & 28.78 & 325.33 & 3.92 & 11.10 & 5.47 & 1.11 & 3.08 & 1.42 & 4.67 & 1.84 \\
DP-DMD & 27.22 & 320.27 & 4.00 & 10.98 & 5.67 & 1.06 & 3.00 & 1.33 & 4.89 & 2.13 \\
Reward Forcing & 26.99 & 315.84 & 4.04 & 10.97 & 5.79 & 1.09 & 3.14 & 1.41 & 5.20 & 2.21 \\
DynaForcing & 27.75 & 326.38 & 4.04 & 10.95 & 5.45 & 1.07 & 2.96 & 1.31 & 4.69 & 2.06 \\
\rowcolor{oursbg}
Routed Forcing & \textbf{25.43} & \textbf{309.92} & \textbf{4.07} & \textbf{10.94} & \textbf{7.80} & \textbf{1.25} & \textbf{3.55} & \textbf{1.58} & \textbf{6.09} & \textbf{2.29} \\
\midrule
\multicolumn{11}{c}{\textbf{\textit{AVSpeech}}} \\
\midrule
Self Forcing & 39.44 & 456.01 & 3.05 & 10.77 & 3.33 & 0.99 & 2.42 & 1.73 & 4.08 & 1.90 \\
DFD & 41.52 & 442.15 & 2.96 & 10.80 & 3.46 & 1.12 & 2.59 & 1.89 & 3.99 & 1.77 \\
DP-DMD & 39.48 & 454.76 & 3.02 & 10.79 & 3.33 & 0.98 & 2.35 & 1.68 & 3.97 & 1.87 \\
Reward Forcing & 39.46 & 453.16 & 3.04 & 10.80 & 3.35 & 1.01 & 2.46 & 1.78 & 4.17 & 1.94 \\
DynaForcing & 39.78 & 460.86 & 3.04 & 10.77 & 3.23 & 0.96 & 2.29 & 1.63 & 3.85 & 1.82 \\
\rowcolor{oursbg}
Routed Forcing & \textbf{38.47} & \textbf{424.58} & \textbf{3.09} & \textbf{10.75} & \textbf{4.85} & \textbf{1.23} & \textbf{2.92} & \textbf{2.01} & \textbf{5.08} & \textbf{2.14} \\
\bottomrule[0.9pt]
\end{tabular}}
\vspace{5pt}
\end{table}

\begin{figure}[t]
  \centering
  \includegraphics[width=\linewidth]{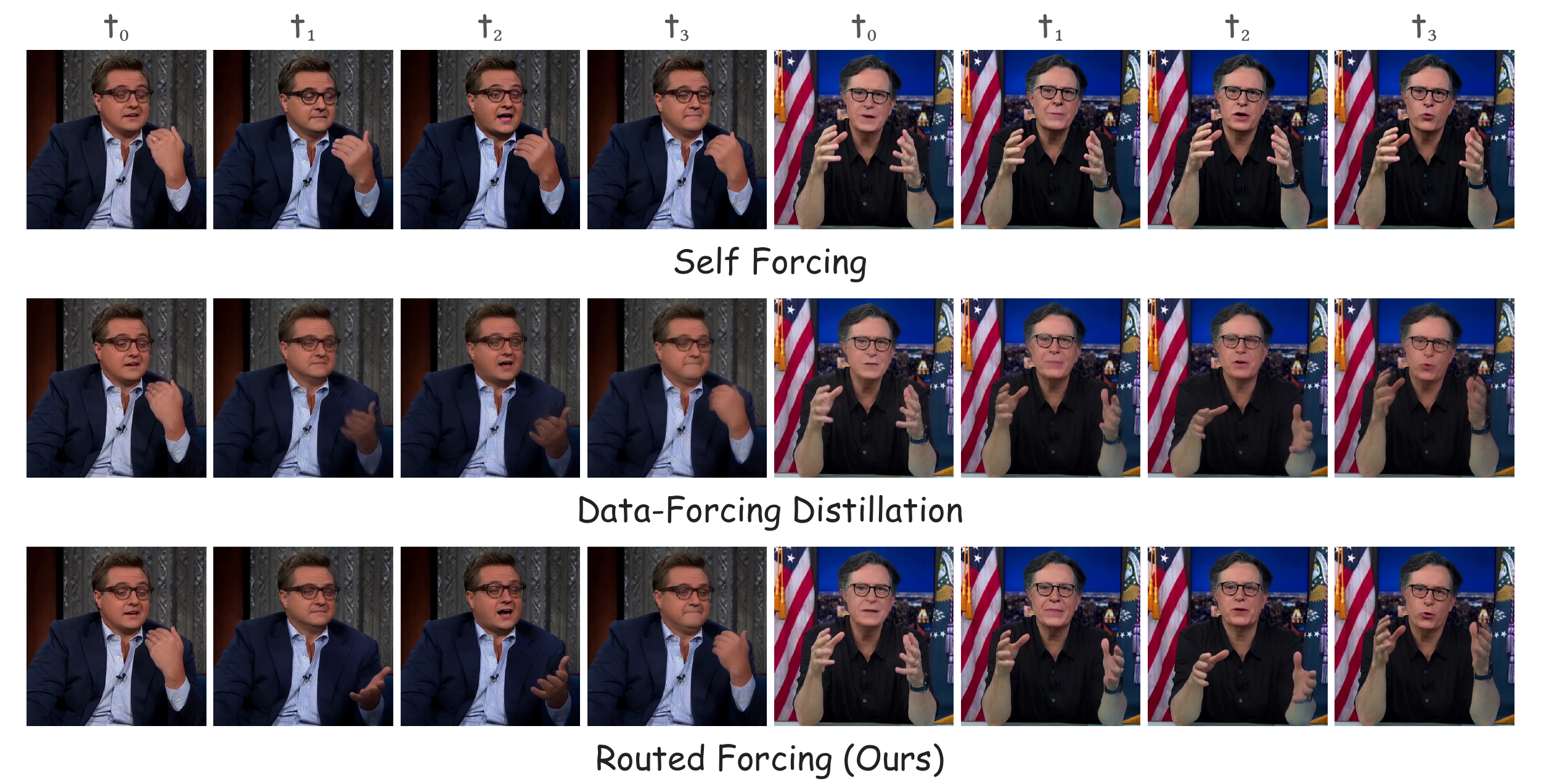}
  \vspace{-1.7em}
  \caption{\textbf{Qualitative comparison of different training strategies.} All methods are conditioned on the same first frame and audio, and each column shows frames at the same timestep. Routed Forcing produces richer dynamics (e.g., diverse gestures) than Self Forcing and better video quality (e.g., clearer hands) than DFD. More examples are provided in Figure~\ref{fig:qualitative-appendix} in the appendix.}
  \label{fig:qualitative-gestures}
\end{figure}

\textbf{Controlled training strategy comparison.}
As shown in Table~\ref{tab:forcing-comparison}, Routed Forcing achieves the best performance across all ten metrics on both datasets. Compared to Self Forcing, Routed Forcing reduces FID by 6.4\%/2.5\% and FVD by 3.4\%/6.9\% on SpeakerVid/AVSpeech. The improvements in dynamics are most pronounced: RAFT-Motion increases by 38.8\%/45.6\% and DINO-Temporal by 16.8\%/24.2\%. All four diversity metrics consistently improve by 7--25\% on both datasets. Lip sync simultaneously remains the best, indicating that the gains in dynamics and diversity do not come at the cost of lip accuracy.

Moreover, DFD fails to consistently improve video quality (FID is worse than Self Forcing on both datasets) and notably degrades lip sync (Sync-C and Sync-D are both worse than Self Forcing). Globally mixing real-trajectory supervision without distinction hurts both lip accuracy and visual quality, whereas Routed Forcing avoids these side effects through semantic and noise-stage routing. DP-DMD applies teacher-target supervision only at the first generation step, but this per-step routing yields negligible improvement compared to noise-level and region routing. Reward Forcing and DynaForcing both remain within the mode-seeking DMD framework: Reward Forcing yields only marginal gains through gradient reweighting, and DynaForcing adds data-anchored rollout with reward weighting but still queries scores on student rollouts.

\textbf{Qualitative comparison of training strategies.}
As shown in Figure~\ref{fig:qualitative-gestures}, Self Forcing produces near-static avatars with highly consistent gestures, a consequence of DMD mode-seeking that collapses the distribution onto simple static modes. DFD improves dynamics by introducing real-video supervision, but applying DFD at low noise stages causes layout mismatch between student-generated and real trajectories, resulting in noticeable hand blur. Routed Forcing restricts DFD to the non-mouth person region at high noise stages via semantic and noise routing, producing richer dynamics (e.g., diverse gestures) than Self Forcing and better video quality (e.g., clearer hands) than DFD. More examples are provided in Figure~\ref{fig:qualitative-appendix} in the appendix.

\begin{table}[!t]
\centering
\caption{\textbf{Comparison with audio-driven avatar models of comparable scale.}}
\label{tab:public-weight-comparison}
\vspace{6pt}
\setlength{\tabcolsep}{4pt}
\renewcommand{\arraystretch}{1.25}
\resizebox{\textwidth}{!}{%
\begin{tabular}{lccccccccc}
\toprule[0.9pt]
\multirow{2}{*}{\raisebox{-0.5ex}{\textbf{Method}}} & \multirow{2}{*}{\textbf{Size}} & \multicolumn{2}{c}{\textbf{Efficiency}} & \multicolumn{2}{c}{\textbf{Video Quality}} & \multicolumn{2}{c}{\textbf{Lip Sync}} & \multicolumn{2}{c}{\textbf{Dynamics} ($\times100$)} \\
\cmidrule(lr){3-4}\cmidrule(lr){5-6}\cmidrule(lr){7-8}\cmidrule(l){9-10}
& & FPS $\uparrow$ & Latency (s) $\downarrow$ & FID $\downarrow$ & FVD $\downarrow$ & Sync-C $\uparrow$ & Sync-D $\downarrow$ & RAFT $\uparrow$ & DINO-T $\uparrow$ \\
\midrule
\multicolumn{10}{c}{\textbf{\textit{SpeakerVid}}} \\
\midrule
LiveTalk & 1.3B & \textbf{24.82} & \textbf{0.33} & 28.77 & 357.28 & 3.10 & 11.62 & 4.77 & 1.07 \\
FantasyTalking & 14B & 0.35 & 232.06 & 33.98 & 364.56 & 2.44 & 12.11 & 3.99 & 0.57 \\
AniPortrait & 2.5B & 0.38 & 211.77 & 33.49 & 570.41 & 3.22 & 11.24 & 4.88 & 1.18 \\
OmniAvatar & 1.3B & 0.97 & 83.44 & 62.83 & 902.13 & 1.46 & 13.74 & 2.27 & \textbf{1.70} \\
\rowcolor{oursbg}
Routed Forcing & 1.3B & \textbf{24.82} & \textbf{0.33} & \textbf{25.43} & \textbf{309.92} & \textbf{4.07} & \textbf{10.94} & \textbf{7.80} & 1.25 \\
\midrule
\multicolumn{10}{c}{\textbf{\textit{AVSpeech}}} \\
\midrule
LiveTalk & 1.3B & \textbf{24.82} & \textbf{0.33} & 40.00 & 491.10 & 2.05 & 11.79 & 2.77 & 1.01 \\
FantasyTalking & 14B & 0.35 & 232.06 & 43.80 & 461.93 & 1.74 & 11.78 & 3.30 & 0.66 \\
AniPortrait & 2.5B & 0.38 & 211.77 & 44.76 & 726.12 & 1.29 & 11.88 & 4.12 & \textbf{2.03} \\
OmniAvatar & 1.3B & 0.97 & 83.44 & 76.70 & 1124.44 & 0.99 & 13.13 & 1.76 & 2.01 \\
\rowcolor{oursbg}
Routed Forcing & 1.3B & \textbf{24.82} & \textbf{0.33} & \textbf{38.47} & \textbf{424.58} & \textbf{3.09} & \textbf{10.75} & \textbf{4.85} & 1.23 \\
\bottomrule[0.9pt]
\end{tabular}}
\vspace{5pt}
\end{table}

\textbf{Comparison with avatar models of comparable scale.}
As shown in Table~\ref{tab:public-weight-comparison}, Routed Forcing shares the same 1.3B causal architecture as LiveTalk, running at 24.82\,FPS (0.33\,s latency) while improving RAFT-Motion by 63\% and FID by 12\% on SpeakerVid, showing that training strategy alone yields substantial gains without architectural change. In terms of video quality and lip sync, Routed Forcing notably outperforms all baselines on both datasets. For dynamics, Routed Forcing leads by a large margin in RAFT-Motion (SpeakerVid 7.80 vs.\ LiveTalk 4.77, FantasyTalking 3.99). OmniAvatar and AniPortrait obtain the highest DINO-Temporal on SpeakerVid (1.70) and AVSpeech (2.03) respectively, but both have lower video quality and lip sync than Routed Forcing.

\subsection{Ablation Study}
\label{sec:ablation}

\textbf{Key component ablation.}
Table~\ref{tab:rf-ablation} isolates the contributions of region routing and noise routing on SpeakerVid. All four conditions branch from the same SF-1000 checkpoint and are trained for 200 steps. \textit{Region routing} increases RAFT-Motion from 5.62 to 6.27 (+11.6\%), while reducing FID from 27.17 to 26.97 and leaving lip sync unchanged, demonstrating that region routing effectively protects mouth accuracy. \textit{Noise routing} yields larger dynamics gains, RAFT-Motion rises from 5.62 to 7.50 (+33.5\%) and diversity improves substantially, but Sync-D increases from 10.97 to 11.02, indicating lip-sync degradation when DFD reaches the mouth.

Full Routed Forcing achieves the best FID (25.43) and RAFT-Motion (7.80) among all conditions, while Sync-D not only avoids degradation but reaches the best value. This synergy shows that region routing confines DFD to the non-mouth person region, eliminating the lip-sync side effects of DFD, while noise routing confines DFD to high noise stages, avoiding interference with low-noise details. Combining region routing and noise routing yields improvements across all four metric groups.

\begin{table}[!t]
\centering
\caption{\textbf{Key component ablation.} We isolate the contributions of region routing and noise routing (Top-50\%). Region routing applies DFD only to the non-mouth person region. Noise routing activates DFD only at high noise stages.}
\label{tab:rf-ablation}
\vspace{6pt}
\setlength{\tabcolsep}{4pt}
\renewcommand{\arraystretch}{1.25}
\resizebox{\textwidth}{!}{%
\begin{tabular}{ccccccccccccc}
\toprule[0.9pt]
\multirow{2}{*}{\textbf{Self Forcing}} & \multirow{2}{*}{\textbf{Region}} & \multirow{2}{*}{\textbf{Noise}} & \multicolumn{2}{c}{\textbf{Video Quality}} & \multicolumn{2}{c}{\textbf{Lip Sync}} & \multicolumn{2}{c}{\textbf{Dynamics} ($\times100$)} & \multicolumn{4}{c}{\textbf{Diversity} ($\times100$)} \\
\cmidrule(lr){4-5}\cmidrule(lr){6-7}\cmidrule(lr){8-9}\cmidrule(l){10-13}
& & & FID $\downarrow$ & FVD $\downarrow$ & Sync-C $\uparrow$ & Sync-D $\downarrow$ & RAFT $\uparrow$ & DINO-T $\uparrow$ & CLIP $\uparrow$ & DINO $\uparrow$ & LPIPS $\uparrow$ & L1 $\uparrow$ \\

\midrule
\ding{51} &  &  & 27.17 & 320.91 & 4.03 & 10.97 & 5.62 & 1.07 & 3.07 & 1.37 & 4.96 & 2.14 \\
\ding{51} & \ding{51} &  & 26.97 & 316.81 & 4.04 & 10.96 & 6.27 & 1.22 & 3.22 & 1.51 & 5.16 & 2.00 \\
\ding{51} &  & \ding{51} & 26.61 & 311.85 & 4.01 & 11.02 & 7.50 & 1.22 & \textbf{3.56} & \textbf{1.60} & 5.92 & 2.22 \\
\rowcolor{oursbg}
\ding{51} & \ding{51} & \ding{51} & \textbf{25.43} & \textbf{309.92} & \textbf{4.07} & \textbf{10.94} & \textbf{7.80} & \textbf{1.25} & 3.55 & 1.58 & \textbf{6.09} & \textbf{2.29} \\
\bottomrule[0.9pt]
\end{tabular}}
\vspace{5pt}
\end{table}

\begin{figure}[t]
  \centering
  \includegraphics[width=\linewidth]{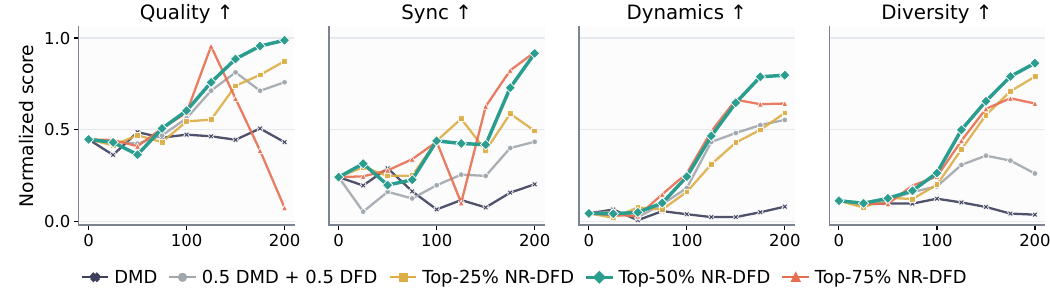}
  \vspace{-1.7em}
  \caption{\textbf{Noise threshold ablation: training curves on SpeakerVid.} Region routing is kept active throughout: all training variants are applied only to the non-mouth person region, with DMD retained for the mouth and background. The y-axis shows the mean normalized score across metrics.}
  \label{fig:rf-threshold-training-curves}
\end{figure}

\textbf{Noise threshold sensitivity.}
Figure~\ref{fig:rf-threshold-training-curves} shows training curves under different noise thresholds. Region routing is kept active throughout: all training variants are applied only to the non-mouth person region, with DMD retained for the mouth and background. Top-$k$\% activates DFD at the highest $k$\% noise stages, and the ``0.5 DMD + 0.5 DFD'' baseline applies 50\% probability DFD mixing across all noise stages, following the original DFD training strategy. The DMD baseline uses no DFD. DMD-only remains stable across all metrics but yields no significant gains. The ``0.5 DMD + 0.5 DFD'' baseline improves over DMD-only but underperforms noise-routed DFD, confirming that restricting DFD to high noise stages is more effective than uniform mixing. Top-75\% exhibits significant quality collapse in later training stages, confirming that DFD at low noise stages harms quality due to layout mismatch between student-generated and real trajectories (consistent with the analysis in \S\ref{sec:noise-routing}). Top-25\% yields improvements between Top-50\% and Self Forcing, as the limited DFD proportion restricts its effect. Top-50\% strikes the best balance across all four metric groups, outperforming both Self Forcing and the original DFD method.

\subsection{Limitations}
\label{sec:limitations}

Semantic-region routing relies on off-the-shelf segmentation models (DeepLabV3 and InsightFace) for mask extraction, so mask accuracy is bounded by the performance of these models.
Routed Forcing is validated on streaming avatar generation. The routing principle may generalize to other video generation tasks such as pose-guided human animation, and this remains to be explored.

\section{Conclusion}

We presented Routed Forcing, which routes DFD by semantic region and noise stage to improve the dynamics and diversity while preserving video quality for streaming avatar generation.
Semantic-region routing confines DFD to the non-mouth person region where region-heterogeneous collapse is most severe, retaining DMD elsewhere. Noise-stage routing restricts DFD to high noise stages where cross-trajectory interference is small, with DMD refining details at low noise stages.
In the future, we plan to extend Routed Forcing to a broader range of applications beyond streaming avatar generation, such as pose-guided human animation.

\bibliography{iclr2027_conference}
\bibliographystyle{iclr2027_conference}

\newpage
\appendix
\section{Why the Noise Level Matters for DFD: An Idealized Analysis}
\label{app:theory}
\label{app:noise}

This section analyzes, under idealized assumptions, how the DFD update changes with the noise level at which the scores are queried.
The analysis motivates noise-stage routing (\S\ref{sec:noise-routing}).
Proposition~\ref{prop:selectivity} describes what the real trajectory adds to the DMD update at high noise, and Proposition~\ref{prop:regression} shows that the same supervision approaches a regression onto the paired real video at low noise.
Propositions~\ref{prop:dfd-modes} and~\ref{prop:dmd-modes} compare DFD and DMD when the student has collapsed onto a static mode of the teacher.

\textbf{Setup.}
We fix the conditioning signal $c$ and assume that the teacher distribution equals the real data distribution.
The teacher and the fake score are modeled as exact conditional-mean denoisers $m_{T,t}$ and $m_{F,t}$ for the teacher distribution and the current student distribution, respectively; each maps a noisy observation to the conditional expectation of the clean latent.
For analytical tractability, the teacher is modeled without classifier-free guidance.
All distributions have finite second moments, and we write $\mu_c := \mathbb{E}_T[X \mid c]$ and $\Sigma_c := \operatorname{Cov}_T(X \mid c)$ for the conditional mean and covariance of the teacher.
The student output is $x^s = G_\theta(\xi, c)$ for a random seed $\xi$, and, given $c$, the seed $\xi$, the real video $x^r$, and the shared standard Gaussian noise $\epsilon$ are mutually independent.
Expectations over $x^r$ treat it as a draw from the conditional data distribution; since the training set contains one real video per condition, these expectations describe a population-level model.
We analyze the update after dividing out the time weight $w_t > 0$ and before region masking.

\textbf{Notation.}
Following the shared-noise coupling in \S\ref{sec:noise-routing} (Eq.~\ref{eq:snr-distance}), we write $\delta = x^s - x^r$ and $r_t := \mathrm{SNR}_t = \alpha_t^2 / \sigma_t^2$ with $\alpha_t, \sigma_t > 0$, and define the unweighted updates $u_{\mathrm{DFD}} = g_{\mathrm{DFD}} / w_t$ and $u_{\mathrm{DMD}} = g_{\mathrm{DMD}} / w_t$.
Here $t$ is the level at which the clean student rollout is re-noised for the score queries, which is distinct from the sampling steps of the student.
Since these updates are gradients, the student moves against them; a term of the form $K(x^s - y)$ with $K \succeq 0$ therefore pulls $x^s$ toward $y$.

\subsection{Covariance-Weighted Correction at High Noise}
\label{app:theory-cov}

DFD differs from DMD only in the input of the teacher query, so the two updates share the fake-score term.

\begin{proposition}[Covariance-weighted correction]
\label{prop:selectivity}
Under the setup above, the additional term that DFD introduces over DMD,
\begin{equation}
b_t := u_{\mathrm{DFD}} - u_{\mathrm{DMD}} = m_{T,t}(z_t^s) - m_{T,t}(z_t^r),
\label{eq:bt-def}
\end{equation}
satisfies
\begin{equation}
b_t = A_t \delta, \qquad A_t := r_t \int_0^1 C_t\bigl(z_t^r + s\,\alpha_t \delta\bigr)\, ds \succeq 0,
\label{eq:At-def}
\end{equation}
where $C_t(z) := \operatorname{Cov}_T(X \mid Z_t = z, c)$ is the posterior covariance of the teacher.
For a fixed pair $(x^s, x^r)$ and noise $\epsilon$, $A_t = r_t \Sigma_c + o(r_t)$ as $r_t \to 0$.
\end{proposition}

\begin{proof}
Under the Gaussian observation model $Z_t = \alpha_t X + \sigma_t \epsilon$, the Jacobian of the conditional mean satisfies the conditional covariance identity $\nabla_z m_{T,t}(z) = (\alpha_t / \sigma_t^2)\, C_t(z)$~\citep{dytso2023conditional}.
Since $z_t^s - z_t^r = \alpha_t \delta$, integrating this identity along the segment between the two points gives Eq.~\ref{eq:At-def}, and $A_t \succeq 0$ because every conditional covariance is positive semidefinite.
As $r_t \to 0$ with $x^s$, $x^r$, and $\epsilon$ fixed, the posterior of $X$ given $Z_t$ converges to the conditional prior, and dominated convergence with finite second moments gives $\int_0^1 C_t\bigl(z_t^r + s\,\alpha_t \delta\bigr)\, ds \to \Sigma_c$.
\end{proof}

If the posterior covariance is bounded along the segment, $\|C_t(z)\|_{\mathrm{op}} \le K_c$, then $\|b_t\| \le K_c\, r_t \|\delta\|$, so the correction vanishes as $r_t \to 0$.
Because $A_t$ is a matrix, the correction also depends on direction.
In the high-noise limit, its leading-order action is governed by the eigenspectrum of $\Sigma_c$: directions of larger conditional variance receive larger gains, and zero-variance directions receive none.
Since $\Sigma_c$ is the covariance of the teacher, this weighting reflects the variation that the data admit under $c$, including variation that the student does not currently produce.
At a finite noise level, $A_t$ averages the posterior covariance along the segment and need not share the eigenvectors of $\Sigma_c$.

\textbf{Gaussian example.}
The following example gives a closed form in which the role of the spectrum can be read off; we do not assume elsewhere that video latents are Gaussian.
If $X \mid c \sim \mathcal{N}(\mu_c, \Sigma_c)$ with $\Sigma_c = U \operatorname{diag}(\lambda_i)\, U^\top$, the posterior covariance does not depend on $z$, $C_t = \Sigma_c (I + r_t \Sigma_c)^{-1}$, and $A_t = r_t \Sigma_c (I + r_t \Sigma_c)^{-1}$.
Projecting $b_t$ onto the $i$-th eigenvector $u_i$ gives $u_i^\top b_t = h_i(t)\, u_i^\top \delta$ with
\begin{equation}
h_i(t) = \frac{r_t \lambda_i}{1 + r_t \lambda_i} \in [0, 1).
\label{eq:transfer-gain}
\end{equation}
We call $h_i(t)$ the \emph{transfer gain} of direction $u_i$: the fraction of the trajectory difference along $u_i$ that DFD adds to the DMD update.
The gain increases with $\lambda_i$ at every noise level, and for $\lambda_i > \lambda_j > 0$ the ratio $h_i(t)/h_j(t) = \lambda_i (1 + r_t \lambda_j) / [\lambda_j (1 + r_t \lambda_i)]$ decreases monotonically from $\lambda_i / \lambda_j$ as $r_t \to 0$ to $1$ as $r_t \to \infty$.
A lower SNR therefore makes the transfer more selective and weaker in every direction, and a higher SNR raises the gain of every positive-variance direction toward one.

\textbf{A conditional noise window.}
Suppose that, once the reference frame, audio, and text are fixed, the motion to be recovered lies mainly in directions with a large conditional variance $\lambda_{\mathrm{motion}}$, whereas frame-specific details that should not be copied from the paired video lie in directions with a small conditional variance $\lambda_{\mathrm{detail}}$.
If $\lambda_{\mathrm{motion}} \gg \lambda_{\mathrm{detail}}$, then in the range
\begin{equation}
\frac{1}{\lambda_{\mathrm{motion}}} \lesssim r_t \ll \frac{1}{\lambda_{\mathrm{detail}}}
\label{eq:selective-window}
\end{equation}
the motion gain is of order one, with $h_{\mathrm{motion}} \ge 1/2$ once $r_t \lambda_{\mathrm{motion}} \ge 1$, while $h_{\mathrm{detail}} \approx r_t \lambda_{\mathrm{detail}} \ll 1$.
Within this window, DFD transfers the motion difference between the real and generated videos with an order-one gain and strongly suppresses the cross-trajectory pull along detail directions.
The correspondence between high-variance directions and motion is a task-level assumption.
The observation in \S\ref{sec:observation} that the non-mouth person region admits many plausible head motions, expressions, and gestures for the same speech input is consistent with this assumption.

\subsection{Cross-Trajectory Regression at Low Noise}
\label{app:theory-low}

At low noise, the full DFD update approaches a regression onto the paired real video, for general teacher and student distributions.

\begin{proposition}[Low-noise regression]
\label{prop:regression}
Under the setup above, let $d$ be the latent dimensionality. Then
\begin{equation}
\mathbb{E}\bigl\| u_{\mathrm{DFD}} - (x^s - x^r) \bigr\|^2 \le \frac{4d}{r_t},
\label{eq:mse-bound}
\end{equation}
so $u_{\mathrm{DFD}} \to x^s - x^r$ in mean square as $r_t \to \infty$.
If, in addition, the generator Jacobian $J_\theta := \partial G_\theta(\xi, c) / \partial \theta$ satisfies $\|J_\theta\|_{\mathrm{op}} \le K_J$, then the expected parameter update $\mathbb{E}[J_\theta^\top u_{\mathrm{DFD}}]$ differs by at most $2 K_J \sqrt{d / r_t}$ from the gradient of the paired regression objective
\begin{equation}
\mathcal{L}_{\mathrm{pair}}(\theta \mid c) = \tfrac{1}{2}\, \mathbb{E}_{\xi, x^r \mid c} \|G_\theta(\xi, c) - x^r\|^2
= \tfrac{1}{2}\, \mathbb{E}_{\xi \mid c} \|G_\theta(\xi, c) - \mu_c\|^2 + \tfrac{1}{2} \operatorname{tr} \Sigma_c .
\label{eq:Lpair}
\end{equation}
\end{proposition}

\begin{proof}
The conditional mean minimizes the mean-squared reconstruction error, so its error is at most $d / r_t$, the error of the naive estimator $Z_t / \alpha_t$.
Applying this bound to the student term $m_{F,t}(z_t^s) - x^s$ and to the teacher term $m_{T,t}(z_t^r) - x^r$, and using $\|a - b\|^2 \le 2\|a\|^2 + 2\|b\|^2$, gives Eq.~\ref{eq:mse-bound}.
The gradient of $\mathcal{L}_{\mathrm{pair}}$ is $\mathbb{E}[J_\theta^\top (x^s - x^r)]$, so Jensen's inequality and Eq.~\ref{eq:mse-bound} bound the difference by $K_J \sqrt{4d / r_t}$.
The second form of $\mathcal{L}_{\mathrm{pair}}$ follows from $\mu_c = \mathbb{E}[x^r \mid c]$ and the conditional independence of $\xi$ and $x^r$.
\end{proof}

The objective $\mathcal{L}_{\mathrm{pair}}$ pulls outputs from all seeds toward the same conditional mean.
When matched real videos differ in hand positions, head angles, or body poses, this cross-trajectory averaging can blur structural details in individual generations.
The bound in Eq.~\ref{eq:mse-bound} grows with $d$ and characterizes the limit $r_t \to \infty$.

\subsection{A Teacher with Static and Dynamic Modes}
\label{app:theory-modes}

To compare the teacher queries of DFD and DMD when the student has collapsed onto a static video, we consider a teacher with one static mode and two dynamic modes.
Let
\[
F(\rho, \varphi) = \bigl(\rho,\ \ell\cos\varphi,\ \ell\sin\varphi\bigr) \in \mathbb{R}^3,
\]
where $\rho$ measures how far a video moves away from the static video and $\varphi$ selects the action.
The teacher places mass $w_0$ on the static video $x_0 = F(0, 0)$ and mass $w_\pm = (1 - w_0)/2$ on each dynamic video $x_\pm = F(D, \pm\varphi_*)$ with $0 < \varphi_* < \pi/2$, and the real video is fixed to $x^r = x_+$.
The student is
\[
G_\theta(\xi) = F(\bar\rho + v\xi,\ \bar\varphi + a\xi), \qquad \Pr(\xi = 1) = \Pr(\xi = -1) = \tfrac{1}{2},
\]
and all four parameters $\theta = (\bar\rho, v, \bar\varphi, a)$ are trained: $\bar\rho$ and $\bar\varphi$ move both seeds together, and $v$ and $a$ separate them.
In particular, the student can place both seeds on the same dynamic video, $\theta = (D, 0, \varphi_*, 0)$.
We measure the motion relative to the static video and the variation across seeds by
\begin{equation}
\begin{aligned}
\mathcal{A}(\theta) &= \mathbb{E}_\xi\|G_\theta(\xi) - x_0\|^2 = \bar\rho^2 + v^2 + 2\ell^2(1 - \cos a\cos\bar\varphi),\\
\mathcal{V}(\theta) &= \operatorname{tr}\operatorname{Cov}_\xi G_\theta(\xi) = v^2 + \ell^2\sin^2 a .
\end{aligned}
\label{eq:modes-metrics}
\end{equation}
Each dynamic video has $\mathcal{A} = D^2 + 2\ell^2(1 - \cos\varphi_*)$, and a student that places its two seeds on $x_+$ and $x_-$ has $\mathcal{V} = \ell^2\sin^2\varphi_*$.

\textbf{Which modes each query returns.}
Write $y = z_t/\alpha_t = x + \epsilon/\sqrt{r}$ with $r = r_t$.
The teacher posterior assigns weight $p_j(y) \propto w_j \exp(-r\|y - x_j\|^2/2)$ to $x_j$, and $m_{T,t} = \sum_j p_j x_j$.
Let $\Delta^2 = \|x_\pm - x_0\|^2 = D^2 + 2\ell^2(1 - \cos\varphi_*)$ and $L_{jk}(y) = w_j e^{-r\|y - x_j\|^2/2} / \bigl(w_k e^{-r\|y - x_k\|^2/2}\bigr)$.
Since $p_j \le \min\{1, L_{jk}\} \le \sqrt{L_{jk}}$, and $\mathbb{E}_\epsilon\sqrt{L_{jk}(x_k + \epsilon/\sqrt{r})} = \sqrt{w_j/w_k}\, e^{-r\|x_j - x_k\|^2/8}$,
\begin{equation}
\mathbb{E}_\epsilon\bigl[p_+ + p_-\bigr](x_0 + \epsilon/\sqrt{r}) \le \sqrt{\tfrac{2(1 - w_0)}{w_0}}\, e^{-r\Delta^2/8},
\qquad
\mathbb{E}_\epsilon\, p_0(x_+ + \epsilon/\sqrt{r}) \le \sqrt{\tfrac{2 w_0}{1 - w_0}}\, e^{-r\Delta^2/8}.
\label{eq:query-weights}
\end{equation}
When $r D^2 \gg 1$, the query at a static student therefore returns the static mode, and the query at the real video returns the dynamic modes.
How the second query divides its weight between $x_+$ and $x_-$ depends on $r\ell^2$, and for $r\ell^2 \ll 1$ each receives close to one half.
In the notation of Eq.~\ref{eq:selective-window}, $D$ plays the role of the motion scale and $\ell$ that of the detail scale.
With $(D, \ell, \varphi_*, w_0, r) = (25, 1, \pi/3, 0.1, 0.1)$, so that $rD^2 = 62.5$ and $r\ell^2 = 0.1$, the query at $x_0$ gives the two dynamic modes a total weight of $1.7\times10^{-4}$, and the query at $x_+$ gives $x_+$ and $x_-$ the weights $0.535$ and $0.465$.

\begin{proposition}[DFD fixed point]
\label{prop:dfd-modes}
Let $\tau = \mathbb{E}_\epsilon\, m_{T,t}(x^r + \epsilon/\sqrt{r}) = (\tau_\rho,\ M\cos\beta,\ M\sin\beta)$ and $B(\kappa) = \mathbb{E}_{Z\sim\mathcal{N}(0,1)}\tanh(\kappa + \sqrt{\kappa}\,Z)$.
With an exact fake score, the unweighted DFD population gradient is
\begin{equation}
\begin{aligned}
g_{\bar\rho} &= \bar\rho - \tau_\rho, &
g_v &= v\,B(\kappa),\\
g_{\bar\varphi} &= \ell M\cos a\,\sin(\bar\varphi - \beta), &
g_a &= \ell\sin a\,\bigl[M\cos(\bar\varphi - \beta) - \ell\,(1 - B(\kappa))\cos a\bigr],
\end{aligned}
\label{eq:dfd-modes-grad}
\end{equation}
where $\kappa = r(v^2 + \ell^2\sin^2 a)$.
If $0 < M < \ell$, the gradient flow $\dot\theta = -g$ has the locally stable fixed point $\theta^* = (\tau_\rho, 0, \beta, a^*)$, where $a^* \in (0, \pi/2)$ is the unique solution of $[1 - B(r\ell^2\sin^2 a^*)]\cos a^* = M/\ell$.
\end{proposition}

\begin{proof}
The student has two equally likely outputs with midpoint $\bar x$ and half difference $\hat x$, where $\|\hat x\|^2 = v^2 + \ell^2\sin^2 a$, so its posterior mean, as a function of $y$, is $\bar x + \hat x\tanh(r\langle y - \bar x, \hat x\rangle)$.
At $y = G_\theta(\xi) + \epsilon/\sqrt{r}$, averaging over $\epsilon$ gives $\bar x + \xi\hat x\, B(r\|\hat x\|^2)$.
The teacher term averages to $\tau$ because $\xi$ and $\epsilon$ are independent.
Multiplying the difference by $\partial G_\theta(\xi)/\partial\theta$ and averaging over $\xi$ gives Eq.~\ref{eq:dfd-modes-grad}.
At a fixed point with $0 < a < \pi/2$, $g_{\bar\rho} = g_v = 0$ give $\bar\rho = \tau_\rho$ and $v = 0$, and $g_{\bar\varphi} = g_a = 0$ give $\bar\varphi = \beta$ together with the equation for $a^*$, since $\bar\varphi = \beta + \pi$ would make the bracket in $g_a$ negative.
The left-hand side of that equation decreases strictly from $1$ to $0$ on $(0, \pi/2)$.
The Jacobian of $g$ at $\theta^*$ is diagonal with the positive entries $1$, $B(r\ell^2\sin^2 a^*)$, $\ell M\cos a^*$, and $-\ell^2\sin a^*\,\frac{\mathrm{d}}{\mathrm{d}a}\bigl\{[1 - B(r\ell^2\sin^2 a)]\cos a\bigr\}\big|_{a = a^*}$.
\end{proof}

For the exact teacher, the angular part of $\tau$ is a convex combination of distinct points on the circle of radius $\ell$, so $M < \ell$ at every finite $r$, and the stable fixed point places the two seeds on the distinct dynamic videos $F(\tau_\rho, \beta \pm a^*)$.
Seed differences grow even after the common action matches the target: at $v = 0$ and $\bar\varphi = \beta$, $\dot a = \ell(\ell - M)\,a + o(a)$.
Up to the static weight bounded in Eq.~\ref{eq:query-weights}, $\tau = \bigl(D,\ \ell\cos\varphi_*,\ \ell\sin\varphi_*\, B(r\ell^2\sin^2\varphi_*)\bigr)$.
When $rD^2 \to \infty$ and $r\ell^2 \to 0$, $\tau \to (D, \ell\cos\varphi_*, 0)$, so $\theta^* \to (D, 0, 0, \varphi_*)$ and the two seeds approach $x_+$ and $x_-$.
As $r \to \infty$ instead, the query at $x^r$ resolves the dynamic mode, $\tau \to x_+$, $M \to \ell$, and $a^* \to 0$, so both seeds copy the real video, in line with Proposition~\ref{prop:regression}.

\begin{proposition}[DMD invariant region]
\label{prop:dmd-modes}
Let $(D, \ell, \varphi_*, w_0, r) = (25, 1, \pi/3, 0.1, 0.1)$.
The box $\mathcal{U} = \{|\bar\rho \pm v| \le 0.6,\ |\bar\varphi \pm a| \le 0.2\}$ is invariant under the flow of the unweighted DMD population gradient at this noise level.
Along every DMD trajectory that starts in $\mathcal{U}$, $\mathcal{A} \le 0.36 + 2(1 - \cos 0.2) < 0.40$ and $\mathcal{V} \le 0.36 + \sin^2 0.2 < 0.40$.
\end{proposition}

\begin{proof}
DMD has the same fake term as DFD and replaces $\tau$ by the teacher output at the student sample, $\bar T(\xi) = \mathbb{E}_\epsilon\, m_{T,t}(G_\theta(\xi) + \epsilon/\sqrt{r}) = x_0 + \sum_{j=\pm} e_j(\xi)(x_j - x_0)$ with $e_j(\xi) = \mathbb{E}_\epsilon\, p_j$.
For $j = \pm$, $p_j \le \min\{1, L_{j0}\}$, and $\log L_{j0}$ is Gaussian in $\epsilon$ with variance $r\Delta^2$ and a mean that is affine in $G_\theta(\xi)$, so $\mathbb{E}_\epsilon\min\{1, L_{j0}\}$ has a closed form that increases with this mean.
Over the outputs allowed by $\mathcal{U}$, the mean is largest at $F(0.6, \pm 0.2)$, and evaluating the closed form there gives $e(\xi) := e_+(\xi) + e_-(\xi) \le 6.5\times10^{-4}$.
For the seed $\xi = 1$, with $\rho_+ = \bar\rho + v$ and $\varphi_+ = \bar\varphi + a$, the flow gives
\[
\begin{aligned}
\dot\rho_+ &= -\bigl(\bar\rho + vB(\kappa)\bigr) + D\,e(1),\\
\dot\varphi_+ &= -\ell^2\bigl[\sin(\bar\varphi + a) - \tfrac{1}{2}\sin 2a\,(1 - B(\kappa))\bigr] + \nu, \qquad |\nu| \le 2\ell^2\sin(\varphi_*/2)\, e(1).
\end{aligned}
\]
For a symmetric binary input at SNR $\kappa$, the minimum mean-squared error equals $1 - B(\kappa)$ and is at most the error $1/(1+\kappa)$ of the linear estimator, so $B(\kappa) \ge \kappa/(1+\kappa)$.
On the face $\rho_+ = 0.6$, $v \in [0, 0.6]$ and $\kappa \ge rv^2$, so $\bar\rho + vB(\kappa) = 0.6 - v + vB(\kappa) \ge 0.6 - v/(1 + rv^2) \ge 0.6\cdot 0.036/1.036 > 0.020$, while $D\,e(1) < 0.017$.
On the face $\varphi_+ = 0.2$, the bracket is at least $\sin 0.2\,(1 - \cos 0.2) > 0.0039$, while $2\ell^2\sin(\varphi_*/2)\, e(1) < 0.00065$.
On the face $\rho_+ = -0.6$, the same bound holds with the opposite sign, and the teacher term $D\,e(1) \ge 0$ also points inward; the face $\varphi_+ = -0.2$ and the faces of the seed $\xi = -1$ follow by symmetry, so the flow points into $\mathcal{U}$ on its whole boundary.
\end{proof}

\textbf{Comparison at the same noise level.}
For the parameters of Proposition~\ref{prop:dmd-modes}, $\tau = (24.9995, 0.5000, 0.0606)$, so $M = 0.504 < \ell$ and Proposition~\ref{prop:dfd-modes} gives $\theta^* = (24.9995, 0, 0.121, 1.001)$.
From the same initialization $\theta_0 = (0.01, 0.02, 0.01, 0.02) \in \mathcal{U}$, the DFD flow converges to $\theta^*$, and the DMD flow settles at $(0.0045, 0.227, 0, 0)$.
Near the static video, the query at the student returns the static mode (Eq.~\ref{eq:query-weights}), so DMD receives almost no supervision toward the dynamic modes.
The query at the real video returns both dynamic modes, so $\tau$ carries their motion, and its angular part lies strictly inside the circle traced by $F$.
DFD therefore moves the common motion to $\tau_\rho$ and separates the seeds between the two actions.
Table~\ref{tab:static-modes} also lists the DFD fixed point at $r = 10$: it has almost the same $\mathcal{A}$ as at $r = 0.1$ and a much smaller $\mathcal{V}$, because the query then resolves the dynamic mode and both seeds copy the real video.

\begin{table}[H]
\centering
\caption{\textbf{Motion and seed variation in the model of Appendix~\ref{app:theory-modes}.} The parameters are $(D, \ell, \varphi_*, w_0) = (25, 1, \pi/3, 0.1)$, and $\mathcal{A}$ and $\mathcal{V}$ follow Eq.~\ref{eq:modes-metrics}. Each dynamic video of the teacher has $\mathcal{A} = 626$, and seeds placed on $x_+$ and $x_-$ give $\mathcal{V} = 0.75$.}
\label{tab:static-modes}
\vspace{6pt}
\small
\setlength{\tabcolsep}{8pt}
\renewcommand{\arraystretch}{1.15}
\begin{tabular}{@{}lccc@{}}
\toprule
Update & $r$ & $\mathcal{A}$ & $\mathcal{V}$ \\
\midrule
DMD, from $\theta_0$ & 0.1 & 0.052 & 0.052 \\
DFD, fixed point & 0.1 & 625.90 & 0.709 \\
DFD, fixed point & 10 & 625.99 & 0.0007 \\
\bottomrule
\end{tabular}
\end{table}

\subsection{Implications for Noise Routing}
\label{app:theory-implications}

The analysis separates three regimes of the query noise level.
(i)~As $r_t \to 0$, the correction $b_t$ of Proposition~\ref{prop:selectivity} vanishes, and the unweighted DFD update approaches the DMD update.
(ii)~At finite high noise, DFD adds a covariance-weighted correction to DMD.
Under the spectral conditions of Eq.~\ref{eq:selective-window}, this correction transfers motion differences with an order-one gain and suppresses detail differences.
In the model of \S\ref{app:theory-modes}, where the student starts near a static mode of the teacher, the query at the real video returns the dynamic modes, so DFD moves the student to them and separates its seeds between two actions, while DMD at the same noise level stays near the static mode.
(iii)~At low noise, DFD approaches the paired regression of Proposition~\ref{prop:regression}, which pulls all seeds toward the same mean and toward the frame-level layout of a different real video; in the model of \S\ref{app:theory-modes}, both seeds then copy the real video.

NR-DFD (Eq.~\ref{eq:nr-dfd}) is designed to use regime (ii) and to leave regime (iii) to DMD, which refines details along the student's own trajectory.
Its single threshold also includes the highest noise levels of regime (i), where the unweighted difference between DFD and DMD vanishes.
The common intuition that high noise levels govern coarse structure and low noise levels govern details applies to these query levels as a statement about which variations each objective emphasizes.
Whether conditional diversity improves in practice is an empirical question, which we address with the multi-seed diversity and temporal dynamics metrics in Tables~\ref{tab:forcing-comparison} and~\ref{tab:rf-ablation}.

\section{Mask Analysis}
\label{app:mask}

\begin{table}[H]
\centering
\caption{\textbf{Mask alignment between real video and student rollout video.} Person masks are extracted via a segmentation model, and mouth masks via facial landmarks.}
\label{tab:mask-alignment}
\vspace{6pt}
\small
\setlength{\tabcolsep}{5pt}
\renewcommand{\arraystretch}{1.15}
\begin{tabular}{@{}cccc@{}}
\toprule
Person coverage (\%) $\uparrow$ &
Person IoU (\%) $\uparrow$ &
Mouth coverage (\%) $\uparrow$ &
Mouth IoU (\%) $\uparrow$ \\
\midrule
94.49 & 89.73 & 43.61 & 32.74 \\
\bottomrule
\end{tabular}
\end{table}

\vspace{-1.5em}
\begin{table}[H]
\centering
\caption{\textbf{Mouth mask source ablation.} Comparison of training with mouth masks from the student rollout video versus ground-truth video.}
\label{tab:mouth-mask-source}
\vspace{6pt}
\small
\setlength{\tabcolsep}{5pt}
\renewcommand{\arraystretch}{1.15}
\begin{tabular}{@{}cccccc@{}}
\toprule
Mouth mask source &
FID $\downarrow$ &
FVD $\downarrow$ &
Sync-C $\uparrow$ &
Sync-D $\downarrow$ &
Step time (s) $\downarrow$ \\
\midrule
Student rollout &
\textbf{25.43} & 309.92 & \textbf{4.07} &
\textbf{10.94} & 63.73 \\
Real video &
25.49 & \textbf{307.91} & 4.03 &
10.99 & \textbf{61.11} \\
\bottomrule
\end{tabular}
\end{table}

\vspace{-1em}
\begin{figure}[H]
\centering
\includegraphics[width=\linewidth]{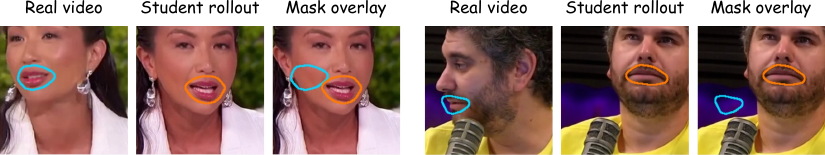}
\vspace{-1.7em}
\caption{\textbf{Comparison of mouth masks between real video and student rollout video.}}
\label{fig:mouth-mask-cases}
\end{figure}

To determine whether masks should be extracted from the student rollout video or the real video, we conduct an ablation study. Table~\ref{tab:mask-alignment} reports the mask alignment between real videos and student rollout videos: person-region coverage reaches 94.49\% with an IoU of 89.73\%, indicating highly consistent person placement, so we directly use person masks from real videos. Mouth coverage is only 43.61\% with an IoU of 32.74\%. As shown in Figure~\ref{fig:mouth-mask-cases}, mouth masks extracted from real videos exhibit noticeable misalignment with the actual mouth position in student rollout videos when the subject turns their head. We further compare training with mouth masks from the student rollout video versus the real video. As shown in Table~\ref{tab:mouth-mask-source}, using student-rollout mouth masks yields better Sync-C and Sync-D, with negligible additional time overhead (63.73 vs.\ 61.11\,s/step). We therefore extract mouth masks from the student rollout video.

\section{More Implementation Details}
\label{app:hyperparams}

Routed Forcing training uses the following setup: both the student and fake score use the Adam optimizer with $\beta_1 = 0.0$, $\beta_2 = 0.999$, no weight decay, no learning rate warmup or decay. The student learning rate is $1{\times}10^{-6}$ and the fake score learning rate is $2{\times}10^{-6}$. The per-GPU batch size is 1 on 8 GPUs with gradient accumulation of 2, yielding an effective batch size of 16. The student updates once every 5 fake score updates (1:5 ratio), i.e., the fake score trains at 5$\times$ the frequency of the student. The classifier-free guidance scale is 4.5. We sample $u\sim\mathcal{U}(0,1)$ and obtain the score timestep using the shift-5 schedule $t=5u/(1+4u)$. ``Top-50\%'' denotes the upper half of this shifted timestep distribution, implemented as $t\geq0.832$.

For dynamics evaluation, we exclude the conditioning frame and sample at 8 FPS. RAFT-Motion averages the largest 5\% of optical-flow magnitudes between consecutive sampled frames, normalized by the image diagonal and temporal interval. DINO-Temporal is the mean cosine distance between DINOv2 features of consecutive sampled frames.

\section{Additional Experimental Results}
\label{app:collapse_curve}
\label{app:l1_diversity}
\label{app:qualitative}

\begin{figure}[H]
\centering
\includegraphics[width=0.75\linewidth]{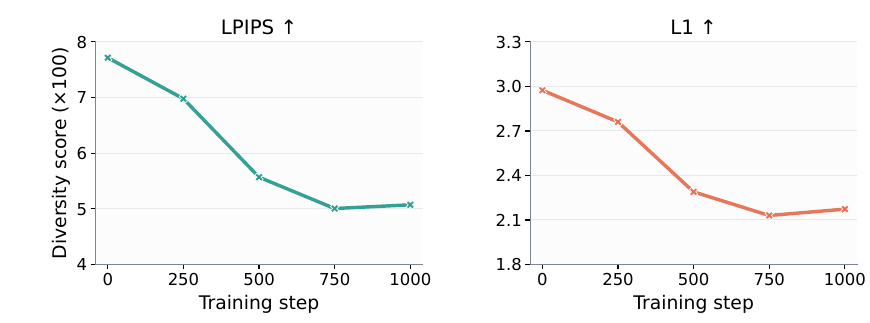}
\caption{\textbf{Diversity decreases during DMD training.} Starting from the LiveTalk checkpoint, we train Self Forcing with DMD for 1,000 steps and measure diversity (mean pairwise LPIPS and $L_1$ across four seeds over 100 SpeakerVid samples) at each checkpoint. Both metrics drop during training, confirming that DMD training itself drives diversity collapse.}
\label{fig:collapse_curve}
\end{figure}

\begin{figure}[H]
\centering
\includegraphics[width=0.75\linewidth]{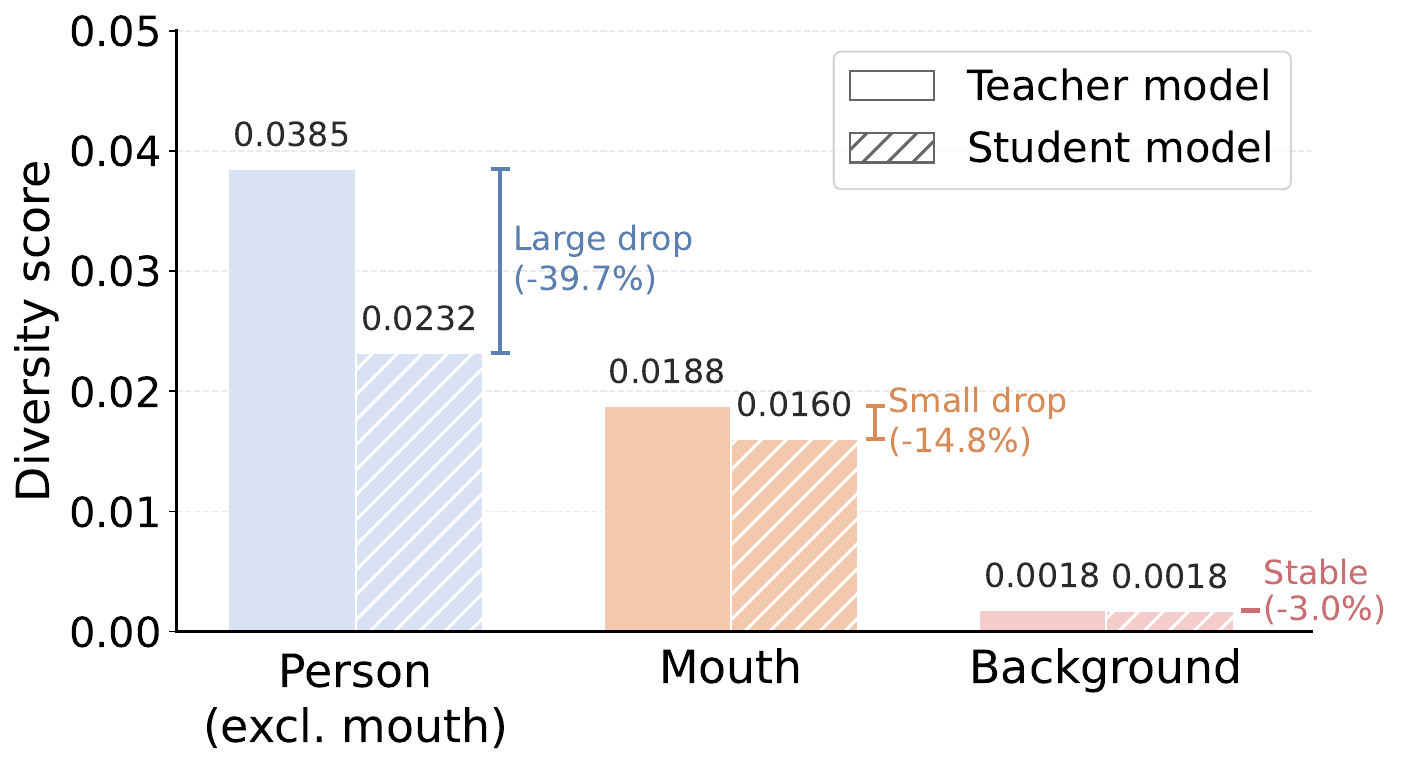}
\caption{\textbf{Regional diversity measured by $L_1$.} Mean pairwise $L_1$ distance across different random seeds for teacher and student models, confirming the same trend as the LPIPS-based analysis in \S\ref{sec:observation}: the person region suffers the largest diversity drop ($-39.7\%$), the mouth region shows a moderate drop ($-14.8\%$), and the background remains stable ($-3.0\%$).}
\label{fig:l1_diversity}
\end{figure}

\begin{figure}[H]
\centering
\includegraphics[width=\linewidth]{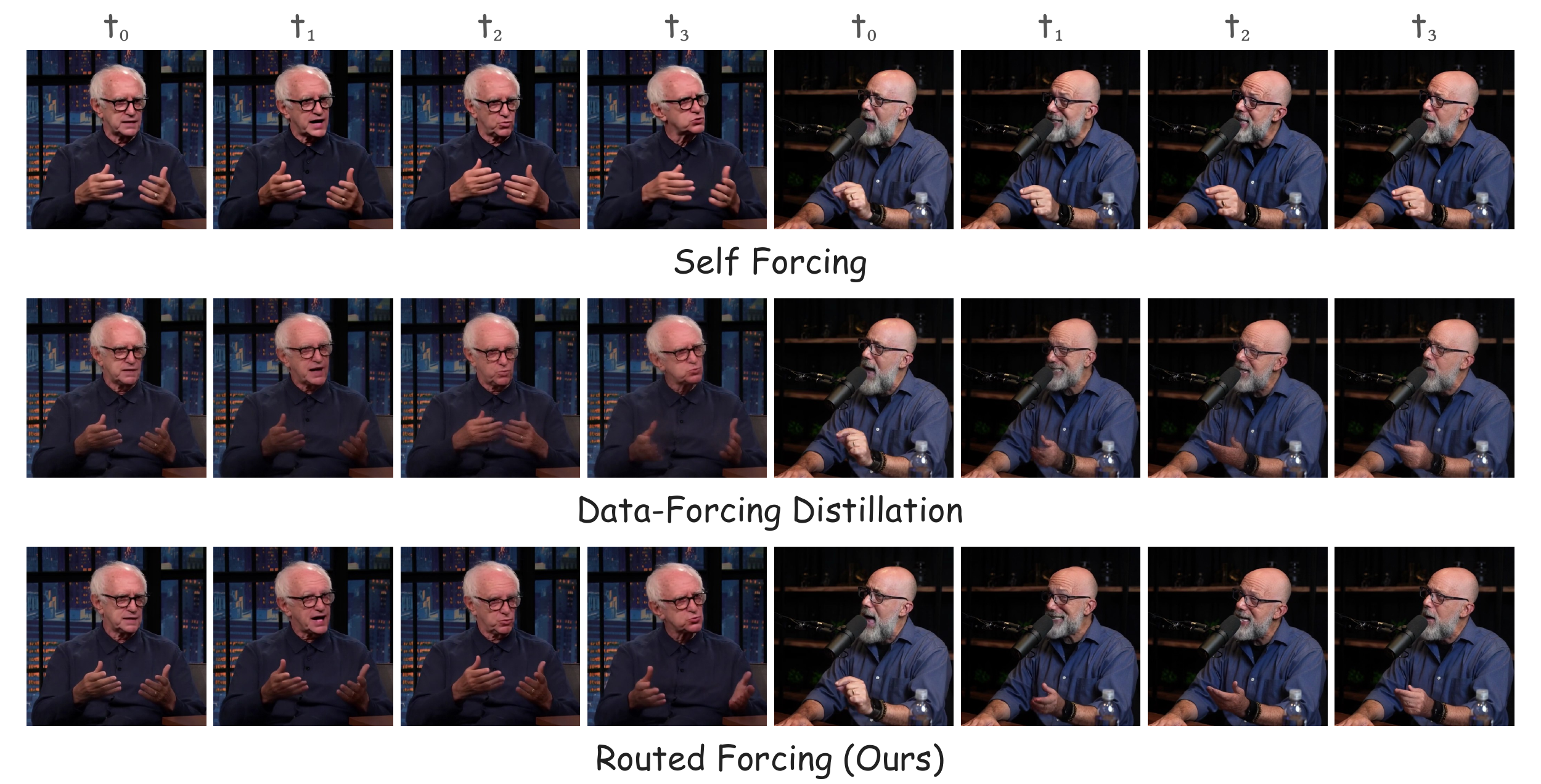}
\vspace{-1.7em}
\caption{\textbf{Additional qualitative comparison of training strategies.} Each column shows frames at the same timestep across methods. Routed Forcing produces more dynamic and diverse motion than Self Forcing while maintaining better visual quality than DFD.}
\label{fig:qualitative-appendix}
\end{figure}

\end{document}

%% file: math_commands.tex
\usepackage{amsmath,amsfonts,amssymb,bm}

\def\eqref#1{equation~\ref{#1}}

\def\1{\bm{1}}

\DeclareMathAlphabet{\mathsfit}{\encodingdefault}{\sfdefault}{m}{sl}
\SetMathAlphabet{\mathsfit}{bold}{\encodingdefault}{\sfdefault}{bx}{n}

